%% file: arxiv_paper_updated.tex
\documentclass{article}

\usepackage{PRIMEarxiv}
\input{includes.tex}
\usepackage[colorlinks,citecolor=blue]{hyperref}
\usepackage{amsthm}
\newtheorem{corollary}{Corollary}
\newtheorem{theorem}{Theorem}
\newtheorem{lemma}{Lemma}
\newtheorem{definition}{Definition}
\newtheorem{prop}{Proposition}

\newcommand*\samethanks[1][\value{footnote}]{\footnotemark[#1]}

\title{Characterizing the Multiclass Learnability of \\Forgiving 0-1 Loss Functions}
\usepackage{times}

\author{
Jacob Trauger\thanks{Equal contribution. Name order determined by the outcome of the 2024 Michigan-Ohio State football game.} \\
    Department of Statistics\\
    University of Michigan\\
    \texttt{jtrauger@umich.com}
\And
Tyson Trauger\samethanks\\
Department of Mathematics\\
The Ohio State University\\
\texttt{trauger.3@osu.edu}
\And
Ambuj Tewari\\
Department of Statistics\\
University of Michigan\\
\texttt{tewaria@umich.edu}
}

\begin{document}

\maketitle

\begin{abstract}%
  In this paper we will give a characterization of the learnability of forgiving 0-1 loss functions in the multiclass setting with effectively finite cardinality of the output and label space. To do this, we create a new combinatorial dimension that is based off of the Natarajan Dimension \citep{natarajan1989learning} and we show that a hypothesis class is learnable in our setting if and only if this Generalized Natarajan Dimension is finite. We also show how this dimension characterizes other known learning settings such as a vast amount of instantiations of learning with set-valued feedback and a modified version of list learning.
\end{abstract}

\section{Introduction}

Classification is one of the most common tasks in machine learning. Within classification, there is normally a split between binary classification (only two possible outputs) and multiclass classification (more than two possible outputs). The theoretical analysis of these settings shares the same split. In binary classification and multiclass classification, the 0-1 loss has been extensively studied and characterized \citep{valiant1984theory,  ben1995characterizations, brukhim2022characterization}. In binary classification, there are 16 possible loss functions $\ell$ such that $\ell(y,y') \in \{0,1\}$, but only two where $\ell(y_1,y_2) \neq \ell(y_1,y_2)$ and $\ell(y_2,y_1) \neq \ell(y_2,y_2)$. Thus, a majority of these loss functions are uninteresting and the two that are interesting are the opposites of each other (the 0-1 loss and $1 - \ell_{0-1}$). 
However, in classification for an output space with cardinality $k$ and a label space with cardinality $t$, there are $(2^t)^k$ different 0-1 loss functions. Most of these loss functions take the form of more ``forgiving" losses (i.e. $\ell(z,y) \in \{0,1\}$, but there can exist many $z_i,y_j$ such that $\ell(z_i,y_j) = 0$). Even more so, many of these settings can have interesting utility; settings such as paraphrase generation in natural language processing \citep{zhou-bhat-2021-paraphrase}, thresholding a metric, ranking with partial feedback \citep{raman2023learnability}, classifying graphs up to isomorphisms (such as is done for drug discovery \citep{yang2024molecule}), and many others can all be seen as allowing for some tolerance on the output. This paper attempts to characterize the learnability of a large class of learning problems with these ``forgiving" multiclass 0-1 loss functions.

The contributions of this paper are as follows:
\begin{enumerate}
    \item We give a new combinatorial dimension based off the Natarajan Dimension \citep{natarajan1989learning}, which we call the Generalized Natarajan Dimension.
    \item We show that the Generalized Natarajan Dimension characterizes the learnability of forgiving 0-1 loss functions in the multiclass setting with minimal extra assumptions on the loss function. Further, we show that the Generalized Natarajan dimension is incomparable to other known dimensions in the classification literature.
    \item We show our characterization implies characterization of various other settings seen in the machine learning literature.
\end{enumerate}

The paper is organized as follows. Section \ref{background} discusses the necessary background information. Section \ref{setup} details the problem set-up. Section \ref{results} shows the results needed to characterize the forgiving 0-1 loss functions along with how the Generalized Natarajan Dimension relates to other dimensions in the literature. Section \ref{apps} gives concrete applications of our characterization. Finally, Section \ref{conclusion} concludes and discusses possible future work.

\subsection{Related Works}\label{related works}
Under the PAC-learning model, binary classification learnability under the 0-1 loss is known to be characterized by the VC-dimension \citep{vapnik1974theory, shalev2014understanding}. 
For multiclass classification, there has also been a further split between finite and infinite label cases. For example, it is known that Empirical Risk Minimization (ERM) is a valid learner in the finite label case \citep{shalev2014understanding}, however, \citet{daniely2015multiclass} has shown that not every ERM is a learner in the infinite label case. The Natarajan Dimension has been known to characterize the learnability of the 0-1 loss in the finite label case for many years now \citep{natarajan1989learning, ben1995characterizations}, but only relatively recently has the DS-Dimension been shown to characterize the learnability of the 0-1 loss under the infinite label case \citep{daniely2014optimal, brukhim2022characterization}.

For related work around PAC-learning general loss functions, \citet{ben1995characterizations} show the finiteness of the Natarajan dimension is sufficient of learnability of any finite-label multiclass loss function.
\citet{david2016statistical} show that sample compression and multiclass learnability are equivalent, but stop short of creating any combinatorial dimension to characterize learnability.
\citet{hopkins2022realizable} show conditions to go from realizable learning to agnostic learning. However, they require assumptions on loss functions that we do not assume. Recently, \citet{hannekerepresentation} were able to upgrade \citet{hopkins2022realizable} to give an agnostic-to-realizable reduction for all multiclass loss functions with $\{0,1\}$ output, however, their rates are of the order of $\frac{1}{\epsilon^3}$ in our setting.

Other ways people have generalized PAC-learnability deal with altering the definition.
\citet{alon2022theory} and \citet{kalavasis2022multiclass} relax PAC-learnability to allow for partial concept classes, which allows the hypotheses to be undefined on parts of the input space.
\citet{pmlr-v291-bressan25b} alter realizable PAC-learnability to allow for more relaxed errors. While our theory allows for different output and label spaces, when they are the same, our realizable setting is a special case of theirs where $z=0$ and when $w = \ell$. Using the reduction from \citet{hannekerepresentation}, we see this does characterize our setting, however, the rates have at least an multiplicative factor of $\frac{1}{\epsilon^3}$, which is worse than ours of $\frac{1}{\epsilon^2}$. We also show their dimension is incomparable with ours; there exists hypothesis classes where our dimension and theirs can be arbitrarily (and in some cases infinitely) larger than each other. Thus, while their characterization along with the agnostic-to-realizable reduction does characterize our setting, our dimension is a more clean, precise, and interpretable characterization of the setting. In their paper they state ``it would be interesting to find characterizations beyond the J-cube dimensions that allow to fully recover known multiclass sample complexity bounds", which we do recover for our setting. 

We also note the similarities of our setting to the setting of list learning. List learning is when the algorithm is allowed to output a list of possible outputs and the loss is given by whether or not the label is present in the learner's outputted list. We can see that, given a multiclass hypothesis $h$ and a loss function $\ell$ whose output is in $\{0,1\}$, we can create an equivalent $h'$ and $\ell'$ where $h'(x) = \{y \mid \ell(h(x),y) = 0 \}$ and \[\ell'(\{y_1,\dots,y_k\},y) = \begin{cases}
0 & y \in \{y_1,\dots,y_k\}\\
1 & \text{otherwise}\end{cases}.\]
This shows that our setting is equivalent to a modified version of list-learning.
List learning was originally studied in \citet{brukhim2022characterization} and was shown to be characterized by the k-DS dimension when output list size is bounded by k \citep{charikar2023characterization}.
However, we note the k-DS dimension does not necessarily characterize our setting. In their setting, one is allowed to create a list learner that outputs \textit{any} bounded list and then it is compared to the best 0-1 learner in the hypothesis class ($\mcH \subset \Y^\X$ but $\mathcal{A}(S) \in \p{2^\Y}^\X$). In our setting, we have to compare to other list learners, not 0-1 learners. We can also only output lists that correspond to sets $\{y' \mid \ell(y,y') = 0 \}$. \citet{hanneke2024list} do show that agnostic/realizable list learning and uniform convergence are all equivalent in the setting where the the algorithms output is being compared to list learners whose lists are size $k$. However, they still allow the algorithm to output any list and our lists are not constrained to output a list of size exactly $k$. Thus, while there is this connection between the current list-learning literature and our setting, they are distinct settings.

For learning with set-valued labels, \citet{liu2014learnability} have studied a similar problem where there is a list with a ``true" label and other ``distractor" labels. They see how well ERM works when trying to find the hypothesis that predicts the true label well. Our setting, where any value in our label set is considered a ``true" label, has been characterized in the online setting \citep{raman2024online}. The batch setting characterization, however, has remained open as far as we are aware. 

\section{Background and Notation}\label{background}

We will first set up the notion of learnability, which is known as \textit{Probably Approximately Correct} (PAC) learnability \citep{valiant1984theory}.
\begin{definition}[Agnostic PAC-learnability]
    Let $\X$ be an input space, $\mcZ$ be our output space, $\Y$ be our label space, $\mcH \subset \mcZ^{\X}$ be a hypothesis class, and $\ell: \mcZ \times \Y$ be a loss function. 
    Then, $\mcH$ is agnostic PAC-learnable with respect to $\ell$ if there exists a learning algorithm $\mathcal{A}$ and a sample function $m: \R \times \R \rightarrow \N$ such that $\forall \epsilon, \delta > 0$, for every distribution $\D$ over $\X \times \Y$, if we have samples $S$ generated from $\D$, $|S| \geq m(\epsilon,\delta)$, then we have:
    \[\Prob{\E_{\D}\br{\ell(\hat{h}(x),y)} \leq \inf_{h \in \mcH}\E_{\D}\br{\ell(h(x),y)} + \epsilon} \geq 1 - \delta\]
    where $\hat{h} = \mathcal{A}(S) \in \mcH$.
\end{definition}
When $\inf_{h \in \mcH}\E_{\D}\br{\ell(h(x),y)} = 0$, it is called the \textit{realizable} case.

One of the most important loss functions is the 0-1 loss function \citep{wang2022comprehensive} and it is a common place to start when studying a new setting (or its equivalent in the setting) \citep{shalev2014understanding, brukhim2022characterization, charikar2023characterization, guermeur2007vc}:
\[\ell_{0-1}(y,y') = \begin{cases}
    0 & y=y'\\
    1 & \text{otherwise}
\end{cases}\]
For multiclass classification with finite labels using the 0-1 loss, it is known that this setting is characterized by the Natarajan dimension \citep{natarajan1989learning, shalev2014understanding}:
\begin{definition}[Natarajan Dimension]
    We say that a hypothesis class $\mcH$ Natarajan shatters a set $S = \{s_1,...,s_n\}$ if $\exists h_1,h_2 \in \mcH$ such that
    \begin{enumerate}
        \item $\forall s_i \in S$, $h_1(s_i) \neq h_2(s_i)$
        \item $\forall S' \subseteq S$ $\exists h \in \mcH$ where $\forall s \in S'$ $h(s) = h_1(s)$ and for all $s\in S \setminus S'$ $h(s) = h_2(s)$
    \end{enumerate} 
    The Natarajan dimension of a hypothesis class $\mcH$, denoted $Ndim(\mcH)$, is the cardinality of the largest set that the hypothesis class Natarajan shatters. 
\end{definition}

    One nice property of the 0-1 loss function is the \textit{identity of indescernibles}.

\begin{definition}[Identity of Indiscernibles]
A loss function $\ell: \Y \times \Y \rightarrow \R_{\geq 0}$ has the identity of indiscernibles property if
\[\forall y_1,y_2\in \Y, \quad \ell(y_1,y_2) = 0 \iff y_1=y_2\]
\end{definition}

This property is seen in many well-used loss functions such as the 0-1 loss and the squared loss. \citet{hopkins2022realizable} show that this property can be used to create an agnostic learner from a realizable learner and has been used to characterize learnability problems \citep{raman2024characterization, raman2024online}. In fact, \citet{hopkins2022realizable} prove this result by relating the loss between $h(x)$ and $h'(x)$ to the classification error $\Prob{h(x) \neq h'(x)}$. Thus, this notion of equality of labels and loss values are extremely interlinked when the loss function satisfies the identity of indiscernibles.

 In this work we will explicitly \textbf{not} assume the identity of indiscernibles or any generalization of it to different output and label spaces. In fact, our assumptions assume hardly any structure at all, allowing our results to generalize to many settings.

We shall now introduce some notation. Given a loss function $\ell: \mcZ \times \Y \rightarrow \{0,1\}$, let $C := C(\ell) = \{(z,y) \mid \ell(z,y) = 0\}$. We will refer to this the equality set. 
We also define $\sigma(z) : = \sigma_{\ell}(z) = \{y \mid \ell(z,y) = 0\}$. Thus, $\sigma(z)$ is all the labels $y$ where $(z,y)$ achieves $0$ loss. We similarly define $\tau(y): = \tau_{\ell}(y) = \{z \mid \ell(z,y) = 0\}$. We will oftentimes omit the $\ell$ subscript from $\sigma$, $\tau$ and $\ell$ parameter for $C$, but there is always an implicit loss function these are referring to. Further, notice that from any one of $C,\ell, \sigma$, or $\tau$, we can reconstruct the others.

Note how $\sigma$, $\tau$ create equivalence relations on $\mcZ$ and $\Y$ by $z_1 \sim_\sigma z_2 \iff \sigma(z_1) = \sigma(z_2)$ and similar for $\Y$ using $\tau$. Let $\sigma(\mcZ)$ and $\tau(\Y)$ be these equivalence classes. We note that $|\sigma(\mcZ)| < \infty \implies |\tau(\Y)| < 2^{|\sigma(\mcZ)|}$ and $|\tau(\Y)| < \infty \implies |\sigma(\mcZ)| < 2^{|\tau(\Y)|}$.

\section{Setup}\label{setup}
Let us have an output space $\mcZ$ and label space $\Y$.  We analyze loss functions $\ell$ that satisfy the following constraints:
\begin{itemize}
    \item $\ell: \mcZ \times \Y \rightarrow \{0,1\}$.
    \item $|\sigma(\mcZ)|< \infty$
    \item $\forall z_1,z_2 \in \mcZ$, $\sigma(z_1)\not\subset \sigma(z_2)$. We note $\subset$ means \textit{strict} subset; we do allow equalities.
\end{itemize}
\citet{hopkins2022realizable} have showed that one can have a realizable learner that is not agnosticly learnable in multiclass settings that do not have the identity of indiscernibles. The example they show relies on the loss function taking 3 values, $0$, $1$, and $c$. Thus, assumption \#$1$ allows our loss functions can only take values of $0$ or $1$, and our results show that this is sufficient to have both agnostic and realizable learnability.

Assumption \#$2$ is needed to allow our generalization of the Natarajan dimension to work. The Natarajan dimension characterizes multiclass learnability for finite output spaces. We can generalize this to infinite output and label spaces, so long as the equivalance class of the output class has finite cardinality. Since our loss can only distinguish over equivalence classes, this can be thought of as bounding the ``effective" output and label space sizes. We will refer to this as \textit{effectively finite}.

Assumption \#$3$ is here as the problem does not make intuitive sense without it. The goal of our learning problem is to minimize the loss. Since we have full access to the loss values, if the outputs that $z_1$ takes no loss on is strictly contained in the outputs $z_2$ takes no loss on (i.e. $\sigma(z_1) \subset \sigma(z_2)$), why would one ever want to output $z_1$? It would always be objectively better to output $z_2$. Thus, without loss of generality, we assume this does not happen, otherwise we can just replace all instances of $z_1$ with $z_2$ in our hypothesis class. In our characterization, this assumption is used in the proof that the finiteness of our dimension is necessary for learning. In Appendix \ref{assumpcounter} we show a counterexample to our main theorem of characterization if this assumption is broken as well. 

\section{The Generalized Natarajan Dimension}\label{results}
With spaces $\mcZ$, $\Y$ and loss function $\ell$ as defined above, the goal of this paper is to characterize PAC-learnability of a hypothesis class. 

First, notice if $z \sim_\sigma z'$ and $y \sim_\tau y'$, then
\[\ell(z,y) = \ell(z',y) = \ell(z',y') = \ell(z,y')\]

This shows that our loss function respects our equivalence relation, and thus $\exists \ell^{\sigma, \tau}: \sigma(\mcZ) \times \tau(\Y) \to \{0,1\}$ that agrees with $\ell$. Thus:
\begin{corollary}\label{equiv_learn_cor}
    Let $\D$ be a distribution on $\X \times \Y$, and $\D^{\sigma, \tau}$ be its quotient onto $\X \times \tau(\Y)$. Further, let $h \in \mcZ^\X$, and $h^\sigma$ be its quotient onto $\left(\sigma(\mcZ)\right)^\X$. Then
    \[\E_\D[\ell(h(x),y)] = \E_{\D^{\sigma, \tau}}[\ell^{\sigma,\tau}(h^{\sigma}(x),y)]\]
\end{corollary}

Therefore we have that $(\X,\mcZ, \Y, \mcH,\ell)$ is equivalent to $(\X,\sigma(\mcZ), \tau(\Y),\sigma \circ \mcH,\ell^{\sigma, \tau})$ as learning problems. One can also very easily convert a $(\X,\mcZ, \Y,\mcH,\ell)$ learning problem into a $(\X, \sigma(\mcZ),\tau(\Y),\sigma \circ \mcH,\ell^{\sigma, \tau})$ learning problem by creating the projection map $p:\mcZ \rightarrow \sigma(\mcZ)$, and then using $p \circ h$ for predictions. This way it is also very easy to recover the original $h \in \mcH$. This can also be easily implemented in code by replacing the quotient spaces $\sigma(\mcZ), \tau(\Y)$ with the set of representatives for the equivalence classes and then mapping each element in $\mcZ, \Y$ to its representative. 

We also note that we use $(\X, \mcZ^C,\Y^C, \mcH^C,\ell^{C})$ as alternative notation to $(\X, \sigma(\mcZ),\tau(\Y),\sigma \circ \mcH,\ell^{\sigma, \tau})$. When the context is clear, we will drop superscripts and abuse notation between the equivalence classes versions of $\mcZ$, $\Y$, $\mcH$, $\ell$ and themselves.

Now, we introduce our new dimension, the Generalized Natarajan dimension, based off \citet{natarajan1989learning}.
\begin{definition}[Generalized Natarajan Dimension]
    We say that a hypothesis class $\mcH$ and loss function $\ell$  Generalized Natarajan shatters a set $S = \{s_1,...,s_n\}$ if $\exists h_1,h_2 \in \mcH$ such that
    \begin{enumerate}
        \item $\forall s_i \in S$, $\sigma(h_1(s_i)) \neq \sigma(h_2(s_i))$
        \item $\forall S' \subseteq S$ $\exists h \in \mcH$ where $\forall s \in S'$ $\sigma(h(s)) = \sigma(h_1(s))$ and for all $s\in S \setminus S'$ $\sigma(h(s)) = \sigma(h_2(s))$
    \end{enumerate} 
    The Generalized Natarajan dimension of a hypothesis class $\mcH$, denoted $GNdim(\mcH, \ell)$, is the cardinality of the largest set that the hypothesis class generalized Natarajan shatters. 
\end{definition}

From this we have the immediate corollary:
\begin{corollary}\label{gn_nat_cor}
Given a hypothesis class $\mcH$ and loss function $\ell$ as studied in this paper, we have
\[GNdim(\mcH,\ell) = Ndim(\sigma \circ \mcH) = GNdim(\sigma \circ \mcH,\ell^{\sigma, \tau})\]
\end{corollary}

The Generalized Natarajan Dimension effectively changes what it means to be equal. One implicit property relied on when a loss function has the identity of indiscernibles is the relationship between equality and loss value. Given $\mcZ = \Y$, we have that $y_1=y_2$ if and only if $\ell(y_1,y_2) = 0$ and $y_1 \neq y_2$ if and only if $\ell(y_1,y_2) > 0$. This allows us to use equality of our labels to infer how the loss will act. This is now no longer the case; we can have $y_1 \neq y_2$ but still have $\ell(y_1,y_2) = 0$. Because of this, our dimension needs to explicitly use the loss values instead of being able to use the labels by proxy.

In Theorem \ref{char_thm} we show the Generalized Natarajan dimension characterizes our setting. Thus, the proxy for equality needed is for the equivalency of the set of labels where $z_1$ and $z_2$ achieve 0 loss for (i.e. $\sigma(z_1) = \sigma(z_2)$). This is a very strict type of equality. One can think of a setting where $|\mcZ|$ is very large (but finite) with $\sigma(z_1)$ and $\sigma(z_2)$ being almost equal. However, since they are not exactly equal, this dimension does not get any smaller than the Natarajan dimension. This is quite counter-intuitive as one would expect learnability to be much easier in this setting due to the vast amount of values that will output $0$ loss, but it is not as you can always make a distribution that puts a lot of mass where $\sigma(z_1)$ and $\sigma(z_2)$ differ. 

\subsection{Characterizing Learnability of Forgiving 0-1 Loss Functions}

With this setup, we state our main theorem:
\begin{theorem}\label{char_thm}
    A $(\X, \mcZ, \Y, \mcH,\ell)$ learning problem is PAC-learnable if and only if $GNdim(\mcH,\ell) < \infty$.
\end{theorem}
We prove this through showing the necessity of $GNdim(\mcH,\ell) < \infty$ in Lemma \ref{nfl} and then showing sufficiency in Lemma \ref{suf_lemma}.

The proof outline is to start by reducing the problem to the equivalent learning problem. Once this reduction is done, the only if direction is proven through an alteration of the proof of the No-Free-Lunch Theorem \citep{shalev2014understanding}. The if direction comes from the bounding the VC-dimension of the loss class by the Generalized Natarajan dimension. This then shows that ERM is a learner and  gives an upper bound on the sample complexity.


\begin{lemma}\label{nfl}
    Let us have a learning problem $(\X, \mcZ, \Y,\mcH,\ell)$ as has been studied in this paper. Then, $(\X,\mcZ, \Y,\mcH,\ell)$ is PAC-learnable $\implies$  the Generalized Natarajan Dimension of $\mcH$ is finite.
\end{lemma}
\begin{proof}[Proof Sketch]
    We show that $GNdim(\mcH, \ell) = \infty \implies (\X,\mcZ,\Y,\mcH,\ell)$ is not learnable. To do this, we modify the No Free Lunch Theorem to change the realizable distributions we are looking over. In the No Free Lunch Theorem for the 0-1 loss, we have a shattered set $\X$ and create distributions over $\X$ where the labels are labeled by the $h_i\in \mcH$ that witness the shattering of $\X$. This does not work in our setting for two reasons: one is that our output space and label space need not be the same and the second is, even if they were, we can have $h_i(x) \neq f(x)$ but still have $\ell(f(x), h_i(x)) = 0$. 
    To fix this, we work with equivalent learning problem $(\X,\mcZ^C,\Y^C,\mcH^C,\ell^C)$ to make sure if $h,h' \in \mcH^C$, $h(x) \neq h'(x)$, then $\exists y$ such that $\ell(h(x),y) \neq \ell(h_c'(x),y)$. Then, we crucially change the labels of our distributions to be uniform over $\sigma(h_i(x))$. The rest of the proof is then modified to work with these changes. Since we have showed this for the equivalent learning problem, by Corollaries \ref{equiv_learn_cor} and \ref{gn_nat_cor}, we have this holds for the original learning setting as well.
\end{proof}
We give the full proof of the above in Appendix \ref{lbproofs}.

Since we have shown necessity the of $GNdim(\mcH,\ell) < \infty$, we now show sufficiency.

\begin{lemma}\label{suf_lemma}
    Let us have a learning problem $(\X, \mcZ, \Y,\mcH,\ell)$ as has been studied in this paper. Then, $GNdim(\mcH,\ell)$ is finite $\implies$ Learnability of $(\X, \mcZ, \Y,\mcH,\ell)$.
\end{lemma}
We give two proofs of this lemma in Appendix \ref{ubproofs}. Appendix \ref{ubvc} uses the VC-dimension of the loss class and Appendix \ref{uborig} shows a different proof of a special case of when $\mathcal{Z} = \mathcal{Y}$. While this specialized proof also gives a worse sample complexity bound than the VC-dimension proof, it is included as we believe a direct proof through uniform convergence adds intuition.

Given lemmas \ref{nfl} and \ref{suf_lemma}, we have now proved Theorem \ref{char_thm}. We note the sufficiency proof does not rely on the strict subset assumption. Therefore, for any effectively finite learning problem with a loss function whose output is in $\{0,1\}$, we have that finite $GNdim(\mcH,\ell)$ is sufficient for learnability.

To find sample complexity, we use the VC-dimension of the loss class. Since we have now shown that $GNdim(\mcH, \ell)$ characterizes our setting, Lemma \ref{vciffgndimlemma} shows that the VC-dimension of the loss class does as well. From this, we use the well-known sample complexity bounds of binary classification get the following result:
\begin{corollary}\label{ub_cor}
    Given an $(\X, \mcZ, \Y, \mcH, \ell)$ learning problem studied in this paper, we have the following bounds on the agnostic PAC-learning sample complexity:
    \[\Omega\p{\frac{GNdim(\mcH, \ell) + \log(1/\delta)}{\epsilon^2}} \leq m(\epsilon,\delta) \leq O\p{\frac{GNdim(\mcH, \ell)\log\p{|\sigma(\mcZ)|} + \log(1/\delta)}{\epsilon^2}}.\]
\end{corollary}
While we do not make claims on the tightness of these bounds, we can see how this corollary recovers the known bounds when using the $0-1$ loss. While our No Free Lunch Theorem proof might indicate a possible $\max_{z \in \mcZ}|\sigma(z)|$ term in the denominator, we do not believe this to be the case. Note how 
\[1-\ell_{0-1}(y,y') = \begin{cases}
    1 & y=y'\\
    0 & \text{ otherwise }
\end{cases}\]
is an extremely forgiving loss function that has $\max_{y \in \Y}|\sigma(y)| = |\Y| - 1$. The 0-1 loss requires the exact correct output while $1 - \ell_{0-1}$ allows for all outputs but one; however, our sample complexity bounds would also indicate no change in samples required to learn for either loss. Our intuitive explanation of this is that these two loss functions both face the same problem, just in the reverse way. The $0-1$ loss requires discernment of hypotheses through finding the ``correct" labels, while the $1 - \ell_{0-1}$ loss requires discernment of hypotheses through finding the ``incorrect" labels. This can be seen by noting that empirical risk maximization of the $0-1$ loss is the same as empirical risk minimization on $1-\ell_{0-1}$. Since the sample complexity of ERMax and ERMin of the $0-1$ loss would both use uniform convergence arguments, we can see they would get the same sample complexity. Thus, while $1- \ell_{0-1}$ is an extremely forgiving loss function, it is not necessarily any more forgiving to learn.

\citet{pmlr-v291-bressan25b} state that a class is not learnable when there is at least a pair of outputs that are difficult to distinguish. Our results show that, while true, this does not tell the full picture. Two outputs being indistinguishable is also the only way we can get a smaller upper bound for the sample complexity as well. In fact, over all learning settings that satisfy our assumptions, the only way we can have our hypothesis class go from unlearnable to learnable is through changing the loss function to have more equivalent outputs.

However, this can be a blessing or a curse. If a loss function $\ell$ and hypothesis class $\mcH$ pair has Generalized Natarajan dimension $d$, then creating an $\ell'$ such that $C(\ell) \subset C(\ell')$ can lead to a larger Generalized Natarajan dimension, and thus possibly a larger sample complexity required to learn. We show an example of this, albeit contrived, in the proof of Proposition \ref{incomp_above}. 

We can see from above then how ``forgiveness" of a loss function is hypothesis class-dependent. From the perspective of the loss functions, we see that, for a given hypothesis class, all loss functions that induce the same equivalence classes on $\mcZ$ will have the same Generalized Natarajan dimension. For example, if a loss does not induce any equivalent outputs, then, by Corollary \ref{equiv_learn_cor}, we have that our Generalized Natarajan dimension of this hypothesis class/loss pair is equal to the Natarajan dimension of the hypothesis class. Thus, even if a loss function looks to be more forgiving, in terms of learnability it might not be. Therefore, the colloquial meaning of a ``forgiving" loss function and what loss functions are truly more forgiving are at odds with each other. 

\subsection{Relation to Other Dimensions}
In Corollary \ref{vciffcor} of the appendix we show that the VC-dimension of the loss class also characterizes our setting. The fact that the Generalized Natarajan dimension and the VC-dimension of the loss class both rely on both the hypothesis class and the loss function is crucial to them being able to characterize the setting. Below, we show that any dimension that relies solely on comparing hypotheses through the normal notions of ``shattering" will not characterize our setting.

Here we give the incomparability results of the Natarajan Dimension and the $d_J$ dimension \citep{pmlr-v291-bressan25b}. 
\begin{prop}\label{incomp_below}
    Let $\X$ be our input space and $\Y$, $|\Y| < \infty$ be our label space. There exists a hypothesis class $\mcH \subset \Y^\X$ such that $\forall J \subset \Y$, $Ndim(\mcH) = d_J(\mcH) = \infty$, but there exists a loss function $\ell$ such that $GNdim(\mcH, \ell) = 0$.
\end{prop}
\begin{proof}
    Let $|\X| = \infty$, $\mcH = \Y^\X$ and let $\ell(y,y') = 0$ for all $y,y' \in \Y$. Note that, since $\mcH$ is all functions from inputs to outputs, then, $Ndim(\mcH) = \infty$, and $\forall J \subset \Y$, $d_J(\mcH) = \infty$. However, since our loss partitions $\Y$ to all the same equivalence class, $\sigma\circ \mcH$ contains only 1 element, thus $GNdim(\mcH, \ell) = 0$.
\end{proof}

\begin{prop} \label{incomp_above}
    For all $q \in \N$, there exists a finite input space $\X$, finite output space $\Y$, finite hypothesis class $\mcH$, and loss function $\ell$ such that $\forall J \subset \Y$, $Ndim(\mcH) = d_J(\mcH) = 0$, but $GNdim(\mcH, \ell) = k$.
\end{prop}
\begin{proof}
    Let $q \in \N$, $\X = \{x_1,\dots, x_q\}$, $\Y = \{y_1,\dots, y_{q^{q+1}}\}$, let 
    \[h_i(x_j) = y_{q(i-1) + j},\]
    and let $\mcH = \{h_i \mid i \in [q^q]\}$.
    Notice how $Ndim(\mcH) = 0$ as no hypotheses have any overlapping outputs. In \citet{pmlr-v291-bressan25b}, they state that $Ndim(\mcH) \geq \max_J d_J(\mcH)$, thus $\forall J \subset \Y$, $d_J(\mcH) = 0$. Now, let $\mcH' = [q]^\X$. Note there are $|\mcH'| = q^q$, thus create an arbitrary bijection $f: \mcH' \rightarrow \mcH$. Partition $\Y$ on the following: 
    \[y \sim y' \iff \exists h',h^\dagger \in \mcH', \exists x_i \in \X \text{ where } h'(x_i) = h^\dagger(x_i) \text{ and } f(h')(x_i) = y, f(h^\dagger)(x_i) = y'\]
    Notice this partitions $\Y$ such that each partition is equivalent to a specific $c \in [q]$ for the outputs of $\mcH'$. Thus, let
    \[\ell(y,y') = \begin{cases}
        0 & y \sim y'\\
        1 & \text{ otherwise}
    \end{cases}.\]
    From the definition of the equivalence class, we can see that $\sigma \circ \mcH$ is then equivalent to all functions from $\{x_1,\dots, x_q\}$ to $[q]$, which has a Natarajan dimension of $q$.
\end{proof}
We do note that, due to results in  \citet{ben1995characterizations} and \citet{pmlr-v291-bressan25b}, we have that $GNdim(\mcH, \ell) = \infty \implies Ndim(\mcH) = \infty \implies \max_J d_J(\mcH) = \infty$. Thus, while the Generalized Natarajan dimension can be be infinitely smaller and arbitrarly larger than these dimensions, it can not be infintely larger.

We also note that the same proof examples and ideas as above can be used to show the incomparability of the Generalized Natarajan dimension to the $DS$-dimension \citep{daniely2014optimal, brukhim2022characterization} and the $k$-$DS$-dimension \citep{charikar2023characterization}.

\begin{corollary}
    There exist input spaces $\X, \X'$, output spaces $\Y, \Y'$, hypothesis classes $\mcH \subset \Y^\X, \mcH' \subset {\Y'}^{\X'}$, and loss functions $\ell,\ell'$ such that $GNdim(\mcH, \ell) = 0$ and $DS(\mcH) = $$k$-$DS(\mcH) = \infty$ and  $GNdim(\mcH', \ell') = q \in \N$ and $DS(\mcH') = k$-$DS(\mcH') = 0$ for any $k \in \N$, $k \geq 1$.
\end{corollary}
One can see that $Ndim(\mcH) = \infty \implies DS(\mcH) = \infty$ since all distributions over $\X$ and a finite subset of $\Y$ is included is a subset of distributions over all $\X \times \Y$. Thus, we have as well $GNdim(\mcH, \ell) = \infty \implies DS(\mcH) = \infty$. It is known that for any $k < k'$, $k$-$DS(\mcH) >$$k'$-$DS(\mcH)$ \citep{charikar2023characterization}, thus we posit a conjecture that $GNdim(\mcH, \ell) = \infty$ does \textbf{not} imply $k$-$DS(\mcH) = \infty$ for $k > 1$. This conjecture comes from the ideas stated in Section \ref{related works} where the modified list-learning setting can be both easier and harder than the normal list-learning setting.

\section{Applications of Characterization}\label{apps}
Given the minimal assumptions made on our learning setting, the Generalized Natarajan dimension is able to characterize many machine learning scenarios. In this section we detail some of these settings.
\subsection{Characterizing Learnability of Set Learning}
One application of the Generalized Natarajan dimension is to set learning. 
We define set learning as receiving set-valued feedback and checking whether or not the output is in the set. More precisely, let $\X$ be our input space, let $\mcZ$, $2 < |\mcZ| < \infty$, be our output space, let our label space be $\Y \subset 2^{\mcZ}$, and let us have a hypothesis class $\mcH \subset \mcZ^\X$. The samples we receive are of the form $(x, v) \in \X \times \Y$. Finally, our loss function becomes for any $z \in \Y$ and for any $y \in \Y$:
\[\ell(z,y) = \begin{cases}
    0 & z \in y\\
    1 & \text{otherwise}
\end{cases}\]
We give an example of what this loss function looks like for $\mcZ = \{1,2,3\}$ and $\Y = 2^{\mcZ}\setminus\{\emptyset\}$ in Figure \ref{fig:set_ex}.

\begin{figure}[!h]
    \centering
    \[
\begin{blockarray}{cccccccc}
& \{1\} & \{2\} & \{3\} & \{1,2\} & \{1,3\} & \{2,3\} & \{1,2,3\} \\
\begin{block}{c(ccccccc)}
 1 & 0 & 1 & 1 & 0 & 0 & 1 & 0  \\
 2 & 1 & 0 & 1 & 0 & 1 & 0 & 0  \\
  3 & 1 & 1 & 0 & 1 & 0 & 0 & 0 \\
\end{block}
\end{blockarray}
 \]
    \caption{An example loss matrix for the set learning setup where $\mcZ = \{1,2,3\}$ $\Y = 2^{\{1,2,3\}}\setminus\{\emptyset\}$. Since each $z \in \mcZ$ is only equivalent to themselves, we see this is characterized by the Natarajan dimension. Note this would still be true if $\Y = 2^{\{1,2,3\}}\setminus\{\emptyset, \{1\}, \{2\}, \{3\}\}$ as well.}
    \label{fig:set_ex}
\end{figure}
We note our assumption on $\ell$ would imply that there is no $z,z' \in \mcZ$ such that $z \in y \implies z' \in y$ but $\exists y' \in \Y$ where $z' \in y'$ but $z \not\in y'$. If false, this means there is a value $z$ where it would always be better to output $z'$ instead. Since we assume our loss function is known, this is easily resolved by removing $z$ and having any $h(x) = z$ be changed to $h(x) = z'$ for all $h \in \mcH$ if. Thus, our learning set-up does not cover all possible set learning scenarios, but the ones we do not characterize are misaligned with the goals of the learning problem we are trying to characterize. Thus, we have the following corollary:
\begin{corollary}
    For all instantiations of set learning that meet our assumptions, the learnability of the learning problem is characterized by the Generalized Natarajan dimension.
\end{corollary}

This problem has been characterized in the online setting \citep{raman2024online}, but the batch setting has remained open as far as we are aware until now.

This is a very useful result as many problems can be formulated as a set-learning problem. For example, any setting that is classifying up-to equivalence classes is included under this setting. This is because, for any $y,y' \in \Y$, $\ell(y,y') = \ell([y],[y'])$ where $[y],[y']$ are the equivalence classes $y,y'$ are in, respectively. In the following subsections we give two concrete examples.
\subsubsection{Classifying Graphs up to Isomorphism}
Classifying graphs is well-used in fields such as drug discovery \citep{yang2024molecule}. In cases such as molecules, it does not matter which graph is outputted, so long as it is isomorphic to the correct label graph. It is known that graph isomorphisms form a partition over the graph space, thus we have the following:
\begin{corollary}
    Given an input space $\X$ and output/label space of graphs of with finite amount of edges/nodes, the Generalized Natarajan dimension characterizes learnability over any 0-1 loss function that respects graph isomorphisms (i.e. $\forall y,y',y^\dagger \in \Y$, $y$ isomorphic to $y'$, $\ell(y,y^\dagger) = \ell(y',y^\dagger)$).
\end{corollary}

\subsubsection{Ranking with Partial Feedback}
A ranking learning problem is one where a learner is required to output a permutation of $[k]$ that indicates a ranking of $k$ elements. An example could be outputting what a learner believes to be the ranking of a person's favorite movies. However, one might only care about whether or not the learner got the person's top 10 movies correct. When only the first $p$ rankings are considered, this is known as ranking with partial feedback. 
\citet{raman2023learnability} characterize the learnability of this setting by showing its equivalence to requiring the 0-1 learnability of each index $i \in [p]$ in the ranking. They show for two classes of losses, one that cares only about the \textit{values} in the outputted ranking, $\mathcal{L}(\ell^{@p}_{prec})$, and one that cares about the \textit{values} and the \textit{order} of the ranking, $\mathcal{L}(\ell^{@p}_{sum})$, that learnability is equivalent to learnability of each index in their respective setting. Furthermore, they show that if any loss is learnable in their respective loss class, then every loss in that loss class is learnable. From their findings, one can also see that if a setting is learnable, then ERM is a valid learner. For a more nuanced discussion on their findings, we refer the readers to Sections 3 and 4 of their paper.

Note how ranking on partial feedback creates equivalence classes over the set of rankings. Depending on if $\ell \in \mathcal{L}(\ell^{@p}_{sum})$ or $\ell \in \mathcal{L}(\ell^{@p}_{prec})$, the equivalence classes may or may not depend on the specific order, but they will depend on the values in the top $p$. 
Let 
\[\ell_{0-1}^{sum @p}(y,y') = \begin{cases}
    0 &  y_i = y_i'\quad \forall i \in [p]\\
    1 & \text{otherwise}
\end{cases}\]
and 
\[\ell_{0-1}^{prec @p}(y,y') = \begin{cases}
    0 & \{y_1,\dots, y_p\} = \{y_1',\dots, y_p'\}\\
    1 & \text{otherwise}
\end{cases}.\]
We can see that $\ell_{0-1}^{sum @p} \in \mathcal{L}(\ell^{@p}_{sum})$ and $\ell_{0-1}^{prec @p} \in \mathcal{L}(\ell^{@p}_{prec})$. Notice that, given a ranking hypothesis class $\mcH$, these loss functions both fit the assumptions needed for our characterization. Thus, we have the following corollary:

\begin{corollary}
    The Generalized Natarajan dimension characterizes the learning setting of ranking with partial feedback.
\end{corollary}
This is the first characterization that we are aware of that gives a dimension on the entire hypothesis class and not on individual indices.

\subsection{Characterizing Learnability of Modified List Learning}

As noted in Section \ref{related works}, we characterize a modified version of list learning for many possible set-ups. This version is the ``flip" of set learning where now our hypotheses output lists and we want to minimize $\Prob{y \in h(x)}$. Thus, given possible labels $\Y$, $|\Y| < \infty$, a set of lists $\mcZ \subset 2^\Y$ where $\forall z,z' \in \mcZ$, $z \not\subset z'$, hypothesis class $\mathcal{H} \subset \mcZ^{\X}$ and loss:
\[\ell(\{y_{i_1},\dots,y_{i_n}\}, y) = \begin{cases}
    0 & y \in \{y_{i_1},\dots,y_{i_n}\}\\
    1 & \text{otherwise}
\end{cases}\]
for $|i| = n < |\Y|$, this gives us a list learning problem that satisfies our constraints. Thus, we have the following corollary:

\begin{corollary}
    The version of list learning described above is characterized by the Generalized Natarajan dimension.
\end{corollary}

\section{Conclusion and Future Work}\label{conclusion}
In this paper we have characterized a large set of more forgiving 0-1 loss functions in the effectively finite label multiclass setting. 

Through the properties of our Generalized Natarajan dimension, we can see that the Natarajan dimension will still characterize many of these loss functions as well. Many of these loss functions could be mostly zeros, but if the quotient space of the outputs is the same as the output space, the learnability is the same as in the 0-1 loss. While this is unintuitive, we believe this is due to how stringent PAC-learnability is. Since PAC-learnability requires learnability over all distributions, one can see how, given a hypothesis class and a loss function that does not reduce the output space, one can adversarially create a distribution that puts a lot of mass on the areas where the sets $\sigma(h(x))$ disagree. This will then nullify the ``forgiving" portion of the loss function.

Further work can be done by finding a way to remove the assumption on the loss function to allow for outputs where one output is ``dominated" by the other. Also, seeing whether or not extensions to effectively infinite output/label spaces would mirror the Natarajan/DS dimension split as in the 0-1 loss case would be an interesting result. Finally, since our analysis of sample complexity goes through the VC-dimension of the loss class, there is potential work to rather show the tightness of these bounds or to give tighter bounds that would show more savings in ``forgiveness".


\bibliographystyle{plainnat}
\bibliography{references}

\appendix
\input{appendix_COLT}

\end{document}

%% file: includes.tex
\usepackage{amsmath}
\usepackage{amsfonts}
\usepackage{xcolor}
\usepackage{graphicx} 
\usepackage{comment}
\usepackage{dsfont}
\usepackage{natbib}
\usepackage{algorithm}
\usepackage{algpseudocode}
\usepackage{tikz}
\newcommand{\br}[1]{\ensuremath{\left[#1\right]}}

\newcommand{\p}[1]{\ensuremath{\left(#1\right)}}
\newcommand{\Prob}[1]{\mathbb{P}\p{#1}}
\newcommand{\bbP}{\mathbb{P}}
\newcommand{\E}{\mathbb{E}}
\newcommand{\N}{\mathbb{N}}
\newcommand{\R}{\mathbb{R}}
\newcommand{\mcH}{\mathcal{H}}
\newcommand{\X}{\mathcal{X}}
\newcommand{\Y}{\mathcal{Y}}
\newcommand{\D}{\mathcal{D}}
\newcommand{\mcZ}{\mathcal{Z}}

\newcommand{\Evv}[2]{\mathop{\mathbb{E}}_{#1}\left[#2\right]}

\usepackage{blkarray}

%% file: appendix_COLT.tex
\section{Proving the Relationship Between $VC(\ell\circ \mcH)$ and $GNdim(\mcH, \ell)$}\label{vciffgndimsection}
Here we show that the VC-dimenion of the loss class and the Generalized Natarajan Dimension are bounded by each other.
\begin{lemma}\label{vciffgndimlemma}
Given a $(\X, \mcZ, \Y, \mcH, \ell)$ learning problem, we have
\[GNdim(\mcH, \ell) \leq VC(\ell \circ \mcH) \leq  4.67 GNdim(\mcH, \ell)\log_2(|\sigma(Z)| + 1).\]
\end{lemma}
\begin{proof}
    First we will prove $GNdim(\mcH, \ell) \leq VC(\ell \circ \mcH)$. Let $S$, $|S|=n$ GN-shatter $\mcH, \ell$. Then $\exists f_1,f_2 \in \mcH$ such that $\forall x \in S$, $\sigma(f_1(x)) \neq \sigma(f_2(x))$ and $\forall T \subset S$, $\exists h_T \in \mcH$ where $\forall x \in T$, $\sigma(h_T(x)) = \sigma(f_1(x))$ and $\forall x \in S\setminus T$, $\sigma(h_T(x)) = \sigma(f_2(x))$.
    Let $S' = \{(x,y) \mid x \in S, y \in \sigma(f_1(x)) \triangle \sigma(f_2(x))\}$ where $\triangle$ denotes the symmetric difference of the two sets. Using $f_1, f_2$ and $h_T$ as above we can see that $\ell \circ \mcH$ is VC-shattered by $S'$.

    Now we will prove $VC(\ell \circ \mcH) \leq 4.67 GNdim(\mcH, \ell)\log_2(|\sigma(Z)| + 1))$. It is easy to see that for any set $S \subset \mathcal{X} \times \mathcal{Y}$ we have $|(\ell \circ \mathcal{H})(S)| \leq |\sigma(\mathcal{H})(S_{\mathcal{X}})|$. This is because if $\ell \circ h$ and $\ell \circ h'$ differ on a point, then $h,h'$ must be different on that point as well.

    Note how $\sigma(\mathcal{H})$ only has $|\sigma(\mcZ)|$ different possible outputs. By \citet{ben1995characterizations}, we have for $m$ samples:
    \[|\sigma \circ \mcH| \leq \p{\frac{me(|\sigma(\mcZ)|+1)^2}{2Ndim(\sigma \circ \mcH)}}^{Ndim(\sigma \circ \mcH)}\]
    Which, by \ref{gn_nat_cor}, we have  \[\p{\frac{me(|\sigma(\mcZ)|+1)^2}{2Ndim(\sigma \circ \mcH)}}^{Ndim(\sigma \circ \mcH)} = \p{\frac{me(|\sigma(\mcZ)|+1)^2}{2GNdim(\mcH, \ell)}}^{GNdim(\mcH, \ell)}\]
    Thus, for a sample of size $VC(\ell \circ \mcH)$ to shatter $\ell \circ \mcH$, we have:
    \[2^{VC(\ell \circ \mcH)} \leq |\ell \circ \mcH| \leq |\sigma \circ \mcH| \leq \p{\frac{VC(\ell \circ \mcH)e(|\sigma(\mcZ)|+1)^2}{2GNdim(\mcH, \ell)}}^{GNdim(\mcH, \ell)}\]
    This is the same inequality set-up as seen in \citet{ben1995characterizations}, thus we get the same bounds, except we substitute in $GNdim(\mcH, \ell)$ and $|\sigma(\mcZ)|$ to get 
    \[VC(\ell \circ \mcH) \leq 4.67 GNdim(\mcH, \ell)\log_2(|\sigma(Z)| + 1)\]
\end{proof}

Given Thoerem \ref{char_thm} and our assumption of $|\sigma(\mcZ)| < \infty$ we have the following corollary:

\begin{corollary}\label{vciffcor}
    Given a $(\X, \mcZ, \Y, \mcH, \ell)$ learning problem studied in this paper, we have
    \[(\X, \mcZ, \Y, \mcH, \ell) \text{ is learnable } \iff VC(\ell \circ \mcH) < \infty\]
\end{corollary}

\section{Finite $GNdim(\mcH, \ell)$ is Necessary for Learnability Proof}\label{lbproofs}
\subsection{Proof by Adapting The No Free Lunch Theorem}\label{lbnfl}
\begin{proof}
    First, by Corollaries \ref{equiv_learn_cor} and \ref{gn_nat_cor}, we have an equivalent learning problem $(\X, \mcZ^C, \Y^C,\mcH^C,\ell^C)$ where $GNdim(\mcH,\ell) = Ndim(\mcH^C)$. We shall focus on this equivalent problem. Further, we will remove the $C$ superscript for ease of presentation at times.
    
    We show the contrapositive. Suppose $GNdim(\mathcal{H},\ell) = \infty$. Then, by shattering, for any $m$ we can find a set of $2m$ points where $f_1,f_2,h \in \mathcal{H}$ such that $\sigma(f_1(x)) \neq \sigma(f_2(x))$ for all $x \in \{x_1,\dots, x_{2m}\}$ but $\sigma(h(x)) = \sigma(f_1(x))$ for $x \in \{x_1,\dots,x_m\}$ and $\sigma(h(x)) = \sigma(f_2(x))$ for the others. Let $T$ = $2^{2m}$, and ${h_1,...,h_T}$ be all of the functions that are guaranteed to exists by the shattering condition.
    
    Let us define the distribution $D_{f_1}$ as follows:
    \begin{align*}
        &\bbP_{D_{f_1}}(x) = \begin{cases}
        1/2m & x \in \{x_1,\dots, x_{2m}\}\\
        0 & \text{otherwise}
    \end{cases}\\
    &\bbP_{D_{f_1}}(y \mid x) =\begin{cases}
        1/|\sigma(f_1(x))| & y \in \sigma(f_1(x))\\
        0 & \text{otherwise}
    \end{cases}
    \end{align*}
    Notice that $\bbP_{D_{f_1}}$ is a uniform distribution over our $2m$ inputs, and the output values are also uniformly distributed over the possible values where $f_1$ would get $0$ loss on. Let us define $D_i$ the same as above, but for $h_i$ instead of $f_1$. Note each of these distributions are realizable under $\mathcal{H}$. 
    
    Let $S_1^i,...,S_k^i$ denote all of the samples of size $m$ sampled from $D_i$. Since each sample of size $m$ is equally likely to be choosen, we have that
       \[\E_{S\sim D_i}[L_{D_i}(A(S))] = \frac1k \sum_{j=1}^k L_{D_i}(A(S_j^i))\]
        Thus we get
        \begin{align*}
            &\max_{i \in [T]} \E_{S\sim D_i}[L_{D_i}(A(S))] = \max_{i \in [T]}\frac1k \sum_{j=1}^k L_{D_i}(A(S_j^i)) \geq\\
            &\frac1T\sum_{i=1}^T\frac1k \sum_{j=1}^k L_{D_i}(A(S_j^i)) = \\
            &\frac1k \sum_{j=1}^k\frac1T\sum_{i=1}^T L_{D_i}(A(S_j^i)) \geq\\
            &\min_{j \in [k]}\frac1T\sum_{i=1}^TL_{D_i}(A(S))
        \end{align*}
        Let us fix a $j \in [k]$, and let $\{s_{1},...s_p\} = \mathcal{X}\setminus S_j^i$. Notice $p \geq m$ as $|\mathcal{X}| = 2m$ and $|S_j^i| \leq m$. We then get for any function $f: \mathcal{X} \to \mcZ^C$
        \begin{align*}
            &L_{D_i}(f) = \frac1{2m}\sum_{r=1}^{2m}\Evv{v \sim Unif(\sigma(h_i(x_r)))}{1_{f(x_r) \not\in \tau(v)}} \geq \\
            &\frac{1}{2p}\sum_{r=1}^{p}\Evv{v \sim Unif(\sigma(h_i(s_r)))}{1_{f(s_r) \not\in \tau(v)}}
        \end{align*}
        Thus,
        \begin{align*}
            &\frac1T \sum_{i=1}^T L_{D_i}(A(S_j^i)) \geq \frac1T\sum_{i=1}^T\frac{1}{2p}\sum_{r=1}^{p}\Evv{v \sim Unif(\sigma(h_i(s_r)))}{1_{A(S_j^i)(s_r) \not\in \tau(v)}} =\\
            & \frac{1}{2p}\sum_{r=1}^{p}\frac1T\sum_{i=1}^T\Evv{v \sim Unif(\sigma(h_i(s_r)))}{1_{A(S_j^i)(s_r) \not\in \tau(v)}} \geq\\
            &\frac12 \min_{r\in [p]}\frac1T\sum_{i=1}^T\Evv{v \sim Unif(\sigma(h_i(s_r)))}{1_{A(S_j^i)(s_r) \not\in \tau(v)}}
        \end{align*}
        Fix $r \in [p]$. By partitioning $h_1,...,h_T$ into disjoint tuples $(h_i,h_{i'})$ by $h_i$ is paired with $h_{i'}$ if $h_i(x) \neq  h_{i'}(x)$ iff $x = r$. Thus on these tuples, we have that for exactly at most one of them (without loss of generality, assume its $h_i$)
\[\Evv{v \sim Unif(\sigma(h_i(s_r)))}{1_{A(S_j^i)(s_r) \not\in \tau(v)}} = 0.\]
We have this for at most exactly one because we quotiented by equivalences over $\sigma$ and we assumed no $\sigma$ set is a proper subset of another. Thus there is only one equivalence class that is always correct on $\sigma(h_i(s_r))$ and that is the class $h_i(s_r)$ itself. 
On the other hand, we then know that on $h_{i'}(s_r)$, since it is a different class, there must be at least one value in $\sigma(h_{i'}(s_r))$ that is not guessed by $A(S_j^i))(s_r)$. Thus
\begin{align*}
    &\Evv{v \sim Unif(\sigma(h_{i'}(s_r)))}{1_{A(S_j^i)(s_r) \not\in \tau(v)}} = \frac{|\sigma(h_{i'}(s_r)))\setminus \sigma(A(S_j^i)(s_r))|}{|\sigma(h_{i'}(s_r)))|} \geq \frac{1}{|\sigma(h_{i'}(s_r)))|} \geq \frac1{\max_{z \in \mcZ^C}|\sigma(z)|}
\end{align*}
With this, we see that over every tuple
\[\Evv{v \sim Unif(\sigma(h_i(s_r)))}{1_{A(S_j^i)(s_r) \not\in \tau(v)}} + \Evv{v \sim Unif(\sigma(h_{i'}(s_r)))}{1_{A(S_j^i)(s_r) \not\in \tau(v)}} \geq \frac{1}{\max_{z \in \mcZ^C}|\sigma(z)|}\]
yielding
\[\frac1T\sum_{i=1}^T\Evv{v \sim Unif(\sigma(h_i(s_r)))}{1_{A(S_j^i)(s_r) \not\in \tau(v)}} \geq \frac1T\left(\frac{T}{2}\frac1{\max_{z \in \mcZ^C}|\sigma(z)|}\right) = \frac1{2\max_{z \in \mcZ^C}|\sigma(z)|}\]

Therefore, we get
\begin{align*}
    &\max_{i \in [T]} \E_{S\sim D_i}[L_{D_i}(A(S))] \geq \frac12 \min_{r\in [p]}\frac1T\sum_{i=1}^T\Evv{v \sim Unif(\sigma(h_i(s_r)))}{1_{A(S_j^i)(s_r) \not\in \tau(v)}} \geq\\
    &\frac12 \min_{r\in [p]}\frac{1}{2\max_{z \in \mcZ^C}|\sigma(z)|} = \frac{1}{4\max_{z \in \mcZ^C}|\sigma(z)|}
\end{align*}
Thus, we have shown there exists a realizable distribution where $\mathcal{A}$ does not do well, thereby showing this is not PAC-learnable, which is what we needed to show for the contrapositive.
\end{proof}

\subsection{Sample Complexity Bounds}\label{lbsc}
Given Corollary \ref{vciffcor}, we know that $\ell \circ \mcH$ is learnable if and only if $(\X, \mcZ, \Y, \mcH, \ell)$ is learnable. Notice for any distribution on $\X \times \Y$, there is an equivalent distribution over $\X \times \Y \times \{0\}$. Learnability of the 0-1 loss of $\ell \circ \mcH$ on the latter distribution implies learnability of $\mcH$ with $\ell$ on the former. It is known that $\ell \circ \mcH$ has a sample complexity lower bound of 
\[\Omega\p{\frac{VC(\ell \circ \mcH) + \log(1/\delta)}{\epsilon^2}}\]
samples as it is a binary classifier. Thus, since $GNdim(\mcH, \ell) \leq VC(\ell \circ \mcH)$, we have a lower bound of
\[\Omega\p{\frac{GNdim(\mcH, \ell) + \log(1/\delta)}{\epsilon^2}}.\]

\section{Finite $GNdim(\mcH, \ell)$ is Sufficient for Learnability Proof}\label{ubproofs}
\subsection{Using VC Dimension of Loss Class}\label{ubvc}
\begin{proof}
Suppose we have an $(\X, \mcZ, \Y, \mcH, \ell)$ learning scenario where $GNdim(\mcH, \ell)$ is finite.
By Lemma \ref{vciffgndimlemma}, we have that 
\[VC(\ell \circ \mcH) \leq  4.67 GNdim(\mcH, \ell)\log_2(|\sigma(Z)| + 1).\]
Since $|\sigma(\mcZ)| < \infty$, we know that $VC(\ell \circ \mcH)$ must be finite.
It is known that the upper bound of the sample complexity of $\ell \circ \mcH$ is 
\[O\p{\frac{VC(\ell \circ \mcH) + \log(1/\delta)}{\epsilon^2}}\]
due to it being a binary classifier.
By using the same logic as in the proof of Section \ref{lbsc} we then get an upper bound on the sample complexity needed to learn $(\X, \mcZ, \Y, \mcH, \ell)$ by replacing $VC(\ell \circ \mcH)$ with $4.67 GNdim(\mcH, \ell)\log_2(|\sigma(Z)| + 1)$. Thus we have learnability of $(\X, \mcZ, \Y, \mcH, \ell)$ if $GNdim(\mcH, \ell)$ is finite.

\end{proof}
\subsection{Special Case where $\mathcal{Z} = \mathcal{Y}$}\label{uborig}
\begin{proof}
    Much like Lemma \ref{nfl}, we shall use the $(\X,\Y^C,\mcH^C,\ell^C)$ equivalent learning problem.
    
    Suppose $(\mcH,\ell)$ has finite generalized Natarajan dimension. For each $y \in \Y^C$, put an arbitrary ordering on $\sigma(y)$, and then for all $k > |\sigma(y)|$, let $\sigma(y)_k$ be some label not in $\mathcal{Y}^C$ (from here on out denoted 0). Let the maps $\sigma_i: \Y^C \cup \{0\} \to \Y^C \cup \{0\}$ denote the function $\sigma_i(y) = \sigma(y)_i$, where $\sigma_i(0) = 0\ \forall i$.  We can extend this map (abusing notation) to $\sigma_i: \X \times (\Y^C \cup \{0\}) \to \X\times(\Y^C \cup \{0\})$, $\sigma_i(x,y) = (x,\sigma_i(y))$. Now, given a distribution $D$, let us define a new distribution $D^i := D\circ\sigma_i^{-1}$ by the pushforward of $D$ through $\sigma_i$. Thus $D^i$ can be thought of as the distribution where each $y$ ``becomes" $\sigma(y)_i$. Notice then
    \begin{align*}
        &1 - L_\D(h) = \Evv{(x,y) \sim \D}{1_{h(x) \in \sigma(y)}} = \Evv{(x,y) \sim \D}{\sum_{i=1}^{|\sigma(y)|}1_{h(x) = \sigma(y)_i}} =\\
        &\Evv{(x,y) \sim \D}{\sum_{i=1}^{|\sigma(y)|}1_{h(x) = \sigma(y)_i} + \sum_{i=|\sigma(y)| + 1}^k0} = \Evv{(x,y) \sim \D}{\sum_{i=1}^{|\sigma(y)|}1_{h(x) = \sigma(y)_i} + \sum_{i=|\sigma(y)| + 1}^k1_{h(x) = 0}}=\\
        &\sum_{i=1}^{k}\Evv{(x,y) \sim \D}{1_{h(x) = \sigma(y)_i}} = \sum_{i=1}^{k}\Evv{(x,y') \sim \D^i}{1_{h(x) = y'}} = \sum_{i=1}^k1- L_{\D^i}(h)
    \end{align*}
  
    Where 
    \[L_{\D^i}(h) := \Evv{(x,y)\sim \D^i}{1_{h(x) \neq y}}\]
    is the 0-1 loss for a multiclass learning problem on the distribution of $D^i$, 
    and where the penultimate equality comes from
    
    \begin{align*}
    &D(\{(x,y) \mid   h(x) = \sigma_i(y)\})=\\
    &D(\{(x,y) \mid  h(x) = y' \wedge y' = \sigma_i(y)  \})=\\
    &D(\{(x,y) \mid h(x) = y' \wedge \sigma_i(x,y) = (x,y')\})=\\
    &D\circ \sigma_i^{-1}(\{(x,y') \mid h(x) = y'\}) =\\
    &D^i(\{(x,y') \mid h(x) = y'\})
    \end{align*}
    For the sample loss of a set of size $m$, we can do a similar decomposition.
    \begin{align*}
        &1-L_S(h) = \frac{|\{(y_j,h(x_j)) \in C\}|}{m} = \frac{\sum_{j=1}^m1_{h(x_j) \in \sigma(y_j)}}{m} =\\
        &\frac{\sum_{j=1}^m\sum_{i=1}^k1_{h(x_j) = \sigma(y_j)_i}}{m} = \frac{\sum_{i=1}^k\sum_{j=1}^m1_{h(x_j) = \sigma(y_j)_i}}{m} = \sum_{i=1}^k1-L_{S^i}(h)
    \end{align*}
    Using these decompositions, we get
    \begin{align*}
        &|L_D(h) - L_S(h)| = |L_D(h) -1+1-L_S(h)| =\\
        &\left|-\sum_{i=1}^k1- L_{D^i}(h) + \sum_{i=1}^k1- L_{S^i}(h)\right| = \left|\sum_{i=1}^kL_{D^i}(h) - L_{S^i}(h)\right| \leq \sum_{i=1}^k\left|L_{D^i}(h) - L_{S^i}(h)\right|
    \end{align*}
    Since we have finite generalized Natarajan dimension, by corollary \ref{gn_nat_cor} we have that $\mathcal{H}^C$ is multiclass learnable, and thus has the uniform convergence property with function $m_{\mcH^C}^{UC}$ (\citet{shalev2014understanding}). Since each $L_{D^i},L_{S^i}$ are equivalent to a multiclass problem on a different distribution, pick sample $S$ with $|S| \geq m_{\mcH^C}^{UC}(\epsilon/k,\delta/k)$. 
     By union bound we get:
    \begin{align*}
        &\Prob{\sup_{h \in \mathcal{H}} |L_D(h) - L_S(h)| > \epsilon} \leq \Prob{\sup_{h \in \mathcal{H}} \sum_{i=1}^k\left|L_D^i(h) - L_S^i(h)\right| > \epsilon} \leq\\
        &\Prob{ \sum_{i=1}^k \sup_{h \in \mathcal{H}}\left|L_D^i(h) - L_S^i(h)\right| > \epsilon} \leq \Prob{\bigcup_{i=1}^k \p{\sup_{h \in \mathcal{H}} \left|L_D^i(h) - L_S^i(h)\right| > \epsilon/k}} \leq k\delta/k = \delta
    \end{align*}
Therefore, we have that $\mcH^C$ has the uniform convergence property and thus ERM is a valid learner for the problem.
\end{proof}

\section{Miscellaneous}
\subsection{Example of a Learning Problem where $GNdim(\mcH, \ell) = \infty$ but it is Learnable}\label{assumpcounter}
Let $\X$ be an infinite input space, $\Y = \{1,2,3\}$, $\mcH = \{1,2\}^\X \cup \{3\}^\X$, and let 
\[\ell(1,1) = 0, \quad  \ell(1,2) = \ell(2,1) = 1, \quad \ell(3,1) = \ell(3,2) = \ell(3,3) = 0\]
Note how $GNdim(\mcH,\ell) = VC(\ell \circ \mcH) = \infty$ as we have all functions from $\X$ to $\{1,2\}$, but the algorithm that always chooses $h \in \mcH$ such that $\forall x \in \X$, $h(x) = 3$ will always be a valid PAC-learner.

%% file: references.bib
@inproceedings{brukhim2022characterization,
  title={A characterization of multiclass learnability},
  author={Brukhim, Nataly and Carmon, Daniel and Dinur, Irit and Moran, Shay and Yehudayoff, Amir},
  booktitle={2022 IEEE 63rd Annual Symposium on Foundations of Computer Science (FOCS)},
  pages={943--955},
  year={2022},
  organization={IEEE}
}

@inproceedings{raman2024online,
  title={Online learning with set-valued feedback},
  author={Raman, Vinod and Subedi, Unique and Tewari, Ambuj},
  booktitle={The Thirty Seventh Annual Conference on Learning Theory},
  pages={4381--4412},
  year={2024},
  organization={PMLR}
}

@inproceedings{liu2014learnability,
  title={Learnability of the superset label learning problem},
  author={Liu, Liping and Dietterich, Thomas},
  booktitle={International conference on machine learning},
  pages={1629--1637},
  year={2014},
  organization={PMLR}
}

@inproceedings{charikar2023characterization,
  title={A characterization of list learnability},
  author={Charikar, Moses and Pabbaraju, Chirag},
  booktitle={Proceedings of the 55th Annual ACM Symposium on Theory of Computing},
  pages={1713--1726},
  year={2023}
}

@inproceedings{hopkins2022realizable,
  title={Realizable learning is all you need},
  author={Hopkins, Max and Kane, Daniel M and Lovett, Shachar and Mahajan, Gaurav},
  booktitle={Conference on Learning Theory},
  pages={3015--3069},
  year={2022},
  organization={PMLR}
}

@article{ben1995characterizations,
  title={Characterizations of Learnability for Classes of $\{$O,…, n$\}$-Valued Functions},
  author={Ben-David, S and Cesabianchi, N and Haussler, D and Long, PM},
  journal={Journal of Computer and System Sciences},
  volume={50},
  number={1},
  pages={74--86},
  year={1995},
  publisher={Elsevier}
}

@article{natarajan1989learning,
  title={On learning sets and functions},
  author={Natarajan, Balas K},
  journal={Machine Learning},
  volume={4},
  number={1},
  pages={67--97},
  year={1989},
  publisher={Springer}
}

@inproceedings{alon2022theory,
  title={A theory of PAC learnability of partial concept classes},
  author={Alon, Noga and Hanneke, Steve and Holzman, Ron and Moran, Shay},
  booktitle={2021 IEEE 62nd Annual Symposium on Foundations of Computer Science (FOCS)},
  pages={658--671},
  year={2022},
  organization={IEEE}
}

@article{guermeur2007vc,
  title={VC theory of large margin multi-category classifiers},
  author={Guermeur, Yann},
  journal={The Journal of Machine Learning Research},
  volume={8},
  pages={2551--2594},
  year={2007},
  publisher={JMLR. org}
}

@article{raman2024characterization,
  title={A characterization of multioutput learnability},
  author={Raman, Vinod and Subedi, Unique and Tewari, Ambuj},
  journal={Journal of Machine Learning Research},
  volume={25},
  number={342},
  pages={1--54},
  year={2024}
}

@article{valiant1984theory,
  title={A theory of the learnable},
  author={Valiant, Leslie G},
  journal={Communications of the ACM},
  volume={27},
  number={11},
  pages={1134--1142},
  year={1984},
  publisher={ACM New York, NY, USA}
}

@book{shalev2014understanding,
  title={Understanding machine learning: From theory to algorithms},
  author={Shalev-Shwartz, Shai and Ben-David, Shai},
  year={2014},
  publisher={Cambridge university press}
}

@misc{vapnik1974theory,
  title={Theory of pattern recognition},
  author={Vapnik, Vladimir and Chervonenkis, Alexey},
  year={1974},
  publisher={Nauka, Moscow}
}

@article{daniely2015multiclass,
  title={Multiclass learnability and the ERM principle.},
  author={Daniely, Amit and Sabato, Sivan and Ben-David, Shai and Shalev-Shwartz, Shai},
  journal={J. Mach. Learn. Res.},
  volume={16},
  number={1},
  pages={2377--2404},
  year={2015}
}

@inproceedings{daniely2014optimal,
  title={Optimal learners for multiclass problems},
  author={Daniely, Amit and Shalev-Shwartz, Shai},
  booktitle={Conference on Learning Theory},
  pages={287--316},
  year={2014},
  organization={PMLR}
}

@article{wang2022comprehensive,
  title={A comprehensive survey of loss functions in machine learning},
  author={Wang, Qi and Ma, Yue and Zhao, Kun and Tian, Yingjie},
  journal={Annals of Data Science},
  volume={9},
  number={2},
  pages={187--212},
  year={2022},
  publisher={Springer}
}

@article{david2016statistical,
  title={On statistical learning via the lens of compression},
  author={David, Ofir and Moran, Shay and Yehudayoff, Amir},
  journal={arXiv preprint arXiv:1610.03592},
  year={2016}
}

@article{kalavasis2022multiclass,
  title={Multiclass learnability beyond the pac framework: Universal rates and partial concept classes},
  author={Kalavasis, Alkis and Velegkas, Grigoris and Karbasi, Amin},
  journal={Advances in Neural Information Processing Systems},
  volume={35},
  pages={20809--20822},
  year={2022}
}

@InProceedings{pmlr-v291-bressan25b,
  title = 	 {A Fine-grained Characterization of PAC Learnability},
  author =       {Bressan, Marco and Brukhim, Nataly and Cesa-Bianchi, Nicol{\`o} and Esposito, Emmanuel and Mansour, Yishay and Moran, Shay and Thiessen, Maximilian},
  booktitle = 	 {Proceedings of Thirty Eighth Conference on Learning Theory},
  pages = 	 {641--676},
  year = 	 {2025},
  editor = 	 {Haghtalab, Nika and Moitra, Ankur},
  volume = 	 {291},
  series = 	 {Proceedings of Machine Learning Research},
  month = 	 {30 Jun--04 Jul},
  publisher =    {PMLR},
  url = 	 {https://proceedings.mlr.press/v291/bressan25b.html}
}

@inproceedings{hanneke2024list,
  title={List sample compression and uniform convergence},
  author={Hanneke, Steve and Moran, Shay and Tom, Waknine},
  booktitle={The Thirty Seventh Annual Conference on Learning Theory},
  pages={2360--2388},
  year={2024},
  organization={PMLR}
}

@inproceedings{hannekerepresentation,
  title={Representation Preserving Multiclass Agnostic to Realizable Reduction},
  author={Hanneke, Steve and Meng, Qinglin and Shaeiri, Amirreza},
  booktitle={Forty-second International Conference on Machine Learning},
  year={2025}
}

@article{raman2023learnability,
  title={On the learnability of multilabel ranking},
  author={Raman, Vinod and Subedi, Unique and Tewari, Ambuj},
  journal={Advances in Neural Information Processing Systems},
  volume={36},
  pages={11988--11999},
  year={2023}
}

@article{yang2024molecule,
  title={Molecule generation for drug design: a graph learning perspective},
  author={Yang, Nianzu and Wu, Huaijin and Zeng, Kaipeng and Li, Yang and Bao, Siyuan and Yan, Junchi},
  journal={Fundamental Research},
  year={2024},
  publisher={Elsevier}
}

@inproceedings{zhou-bhat-2021-paraphrase,
    title = "Paraphrase Generation: A Survey of the State of the Art",
    author = "Zhou, Jianing  and
      Bhat, Suma",
    editor = "Moens, Marie-Francine  and
      Huang, Xuanjing  and
      Specia, Lucia  and
      Yih, Scott Wen-tau",
    booktitle = "Proceedings of the 2021 Conference on Empirical Methods in Natural Language Processing",
    month = nov,
    year = "2021",
    address = "Online and Punta Cana, Dominican Republic",
    publisher = "Association for Computational Linguistics",
    url = "https://aclanthology.org/2021.emnlp-main.414/",
    doi = "10.18653/v1/2021.emnlp-main.414",
    pages = "5075--5086"
}
